\documentclass{article}
\pdfoutput=1
\usepackage[accepted]{icml2019}
\usepackage{dsfont}
\usepackage{amsthm}
\usepackage{amsfonts}
\usepackage{amsmath}
\usepackage{amssymb}
\usepackage{url}
\usepackage{comment}
\usepackage{microtype}
\usepackage[pdftex]{graphicx}
\usepackage{subfigure}
\usepackage{booktabs}
\usepackage{epstopdf}
\usepackage{natbib}

\newtheorem{theorem}{Theorem}
\newtheorem{lemma}{Lemma}

\newcommand{\E}{{\mathbb E}}

\usepackage{color}

\DeclareMathOperator*{\argmax}{argmax}
\DeclareMathOperator*{\argmin}{argmin}

\icmltitlerunning{Adaptive Sensor Placement for Continuous Spaces}

\begin{document}

\twocolumn[
\icmltitle{Adaptive Sensor Placement for Continuous Spaces}

\begin{icmlauthorlist}
\icmlauthor{James A. Grant}{pio,sto}
\icmlauthor{Alexis Boukouvalas}{pio}
\icmlauthor{Ryan-Rhys Griffiths}{cam}
\icmlauthor{David S. Leslie}{pio,lan}
\icmlauthor{Sattar Vakili}{pio}
\icmlauthor{Enrique Munoz de Cote}{pio}
\end{icmlauthorlist}

\icmlaffiliation{pio}{PROWLER.io Ltd, Cambridge, United Kingdom}
\icmlaffiliation{sto}{STOR-i Centre for Doctoral Training, Lancaster University, Lancaster, United Kingdom}
\icmlaffiliation{lan}{Department of Mathematics and Statistics, Lancaster University, Lancaster, United Kingdom}
\icmlaffiliation{cam}{Department of Physics, University of Cambridge, Cambridge, United Kingdom}

\icmlcorrespondingauthor{James A. Grant}{j.grant@lancaster.ac.uk}

\vskip 0.3in
]

\printAffiliationsAndNotice{}

\begin{abstract}
    We consider the problem of adaptively placing sensors along an interval to detect stochastically-generated events. We present a new formulation of the problem as a continuum-armed bandit problem with feedback in the form of partial observations of realisations of an inhomogeneous Poisson process. We design a solution method by combining Thompson sampling with nonparametric inference via increasingly granular Bayesian histograms and derive an $\tilde{O}(T^{2/3})$ bound on the Bayesian regret in $T$ rounds. This is coupled with the design of an efficent optimisation approach to select actions in polynomial time. In simulations we demonstrate our approach to have substantially lower and less variable regret than competitor algorithms.
\end{abstract}

\section{Introduction}
\label{sec.intro}

In this paper we consider the problem of adaptively placing sensors to detect events occurring stochastically according to a inhomogeneous Poisson process. This is a problem arising in numerous applications including ecology~\cite{forestry}, and astronomy~\cite{astronomy}. Adaptive sequential decision-making that learns an optimal placement of sensors in response to observations can lead to detecting many more events than fixed policies based on an assumed Poisson process rate function. We study the problem under a simple abstract framework which encompasses many possible practical scenarios, including choosing which hours to operate to maximise customer engagement, or choosing placement of mobile base stations to service as many requests as possible, as well as the classical sensing applications.

Suppose that a decision-maker is tasked with placing a finite number of sensors along an interval. The decision-maker's objective is to maximise, through time, a reward function which trades off the number of events detected with the cost of sensing. At each step, each sensor is tasked with sensing a subinterval, with the cost of sensing depending on the length of the subinterval. Only the events that occur in a sensed subinterval are detected. The decision-maker may update the placement of sensors at regular intervals creating a sequential problem where the decision-maker iteratively places sensors and receives feedback on where events occurred.

The decision-maker therefore faces a classic exploration-exploitation dilemma. In each round they will gather information on what was detected in the sensed regions, and will receive a reward. The most informative action is to sense the entire interval, but this may not be the reward-maximising action due to the cost of sensing. Hence the decision-maker must choose sensor placements to trade off learning about regions where information is insufficient, while also capitalising on information they already have to generate large rewards. This paper develops an algorithm to tackle this problem with the aim of minimising {Bayesian regret}, the difference between the expected reward achieved by constantly selecting an optimal action and the expected reward of actions actually taken, where the expectation is taken with respect to the prior over the reward-generating parameters.

Multi-armed bandits provide models for sequential decision problems, and our problem most closely resembles the continuum-armed or $\mathcal{X}$-armed bandit problem~\cite{Agrawal1995}. In a continuum-armed bandit (CAB) problem a decision-maker sequentially selects points in some $d$-dimensional continuous space and receives reward in the form of a noisy realisation of some unknown (usually Lipschitz smooth) function on the space. Our sensor placement problem can map to this framework by considering that the placement of sensors can be represented by the set of endpoints of the sensors' subintervals. Note, however, that the noise and feedback models in the sensor placement problem are more complex than in previous treatments of CAB models, which have focused on simple numerical reward observations with bounded or sub-Gaussian noise \citep[e.g.][]{Bubeck2011}. In this paper, we handle the added complexities of observing event locations and the heavier-tailed noise of the Poisson distribution.

Our proposed method performs fast Bayesian inference on the rate function, by means of a Bayesian histogram approach \cite{Gugushvili2018}, and makes decisions to trade off exploration and exploitation using Thompson sampling (TS) \citep[e.g.][]{Russo2018}.
\citeauthor{Gugushvili2018}'s approach to nonparametric inference on the continuous action space imposes a mesh structure over the interval, splitting it into a finite number of bins, with the mesh becoming finer as time increases. Inference is then performed over the rate of event occurrence in each bin. 
TS methods select an action in a given round according to the posterior probability that it is optimal. In our approach, this is implemented by sampling bin rates from the simple posterior distributions of \citeauthor{Gugushvili2018}'s model and selecting an optimal action for these sampled rates via an efficient optimisation algorithm described in Section \ref{sec.actionselection}.

We analyse the Bayesian regret of the TS algorithm in this setting using similar techniques to those of \citet{RussoVanRoy2014}. This allows us to derive an $\tilde{O}(T^{2/3})$ upper bound on the Bayesian regret that holds across all possible rate functions with a bounded maximum, and has minimal dependency on the prior used by the TS algorithm. The CAB problem with Poisson noise and event data as feedback is to the best of our knowledge unstudied, however our regret upper bound is encouragingly close to the $\Omega(T^{2/3})$ lower bound on simpler CAB models of \citet{Kleinberg2005}.

The remainder of the paper is structured as follows. We review related work in Section \ref{sec.litreview}, formalise our model and algorithm in Sections~\ref{sec.model}, present the regret analysis in Section~\ref{sec.regret_bound}, and conclude with simulation experiments in Section~\ref{sec.sims}.

\subsection{Principal Contributions}
The principal contributions of this work are: (i) formulation of a new widely applicable model of sequential sensor placement as a CAB;
(ii) the first study of CABs with Poisson process feedback, and use of a new progressive discretisation technique as an approximation to the continuous action space; (iii) an efficient optimisation routine for sensor placement given known event rate; (iv) analysis of the Bayesian regret of a TS approach, resulting in a $\tilde{O}(T^{2/3})$ upper bound; (v) numerical validation of the efficacy of the TS method, and its favourable performance relative to upper confidence bound and $\epsilon$-greedy approaches.

\section{Related Work}
\label{sec.litreview}
The problem of allocating searchers in a continuous space has been studied by \citet{Carlsson2016} under the assumption that the rate of arrivals is known. A first attempt to solve a version of the problem in which the rate must be learned is presented in \citet{Grant2018}, in which the space is discretised to a fixed grid for all time. The objective of our paper is to present the first learning version of the problem for the fully continuous space.

The fixed discretisation version of the problem maps directly to Combinatorial Multi-Armed Bandits (CMAB) \cite{CesaBianchi2012,Chen2016}. This is a class of problems wherein the decision-maker may pull multiple arms among a discrete set and receives a reward which is a function of observations from individual arms. 
In the discretised sensor-placement problem, the individual arms correspond to cells of the grid.
The model remains relevant for the continuous version of the problem, as by using an increasingly fine mesh, we approximate the problem with a series of increasingly many armed CMABs.

The continuum-armed bandit (CAB) model \cite{Agrawal1995} is an infinitely-many armed extension of the classic multi-armed bandit (MAB) problem. There are two main classes of algorithm for CAB problems: discretisation-based approaches which select from a discrete subset of the continuous action space at each iteration, and approaches which make decisions directly on the whole action space. 
Our proposed method belongs to the former class. Early discretisation-based approaches focused on fixed discretisation \cite{Kleinberg2005, Auer2007}, with more recent approaches typically using adaptive discretisations such as a ``zooming'' approach \cite{Kleinberg2008} or a tree-based structure \cite{Bubeck2011,Bull2015,Grill2015} to manage the exploration. Authors who handle the full continuous action space typically use Gaussian process models to capture uncertainty in the unknown continuous function and balance exploration-exploitation in light of this  \cite{Srinivas2009,Chowdhury2017,Basu2018}. As mentioned in Section \ref{sec.intro}, our problem can map into a CAB, but since our information structure is more complex, our action space has dimension greater than 1, and the stochastic components have heavier tails than usual, standard algorithms and results do not apply.

Thompson sampling (TS) is a particularly convenient, and generally effective, method for trading off exploration and exploitation. The critical ideas can be traced as far back as \citet{Thompson1933}, although the first proofs of its asymptotic optimality 
came much later \citep{May2012,Agrawal2012,Kaufmann2012}. Later, similar results were derived for MABs with rewards from univariate exponential families \cite{Korda2013} and in multiple play bandits \cite{Komiyama2015, Luedtke2016}. 
More recently, TS has been studied in the CMAB framework by \citet{Wang2018} and \citet{Huyuk2018} under slightly differing models, but both with bounded reward noise. Both papers demonstrate the asymptotic optimality of TS with respect to the frequentist regret, and we anticipate that these results could be extended to univariate exponential families. However, in both of these works, the leading order coefficients can be highly suboptimal. Therefore, rather than attempt to extend these ideas to CABs, we favour an alternative analysis of the Bayesian regret to get bounds that are of slightly suboptimal order but are more meaningful because of their (relatively) small coefficients. The Bayesian regret is less extensively studied than the frequentist regret. However the bounds that have been derived for the Bayesian regret of TS \cite{RussoVanRoy2014, BubeckLiu2013} are powerful as they do not depend on a specific parameterisation of the reward functions.

\section{Model and solution}
\label{sec.model}

We now formally present our model and solution method.

\subsection{Reward and regret}

In each of a series of rounds $t \in 
\mathds{N}$, $m_t \geq 0$ events of interest arise at locations $X_{t,1},...,X_{t,m_t} \in [0,1]$ according to a non-homogeneous Poisson process with rate $\lambda: [0,1] \rightarrow \mathds{R}_+$.
$U$ sensors are deployed in each round with each sensor observing a distinct subinterval of $[0,1]$; the action space ${\mathcal A}$ consists of the sets of at most $U$ disjoint intervals of $[0,1]$. Let $A_t \subseteq [0,1]$ be the union of the subintervals covered by the sensors in round $t$. An event $X_{t,i}$ is detected if it lies in $A_t$.
The system objective is to maximise the number of detected events while penalised by a cost of operating the sensors. 
The expected reward for playing action $A$ is therefore 
\begin{align*}
    r(A
    ) = \int_{A}(\lambda(x)-C)\,{\rm d}x,
\end{align*}
where $C$ is the cost per unit length of sensing. We define the Bayesian regret of an algorithm to be the expected difference (with respect to the prior on $\lambda$) between the reward achieved when playing the optimal action in each of $T$ rounds and the actions taken by the algorithm:
\begin{equation*}
    BReg(T) = \sum_{t=1}^T \mathbb{E}\left(r(A^*)-r(A_t)\right)
\end{equation*} where $A^* = \arg\max_{A \subseteq \mathcal{A}} r(A)$ is the optimal action on the continuous interval.

\subsection{Inference}

With the Poisson process rate being defined on the continuum $[0,1]$, nonparametric estimation is preferable to a parametric form.
We use the increasingly granular histogram approach of \citet{Gugushvili2018}, since it provides us with fast inference and a concentration rate.
At the beginning of each round $t$ a piecewise-constant estimation of $\lambda$ is considered by counting the number of events to have been observed in each of $K_t$ bins.
The number of bins will be gradually increased as rounds proceed.
To maintain simplicity in the inference and analysis we choose all bins to be of a constant width $\Delta_t=K_t^{-1}$.

We introduce the notation
\begin{displaymath}
B_{k,t} \equiv \bigg[\frac{k-1}{K_t},\frac{k}{K_t}\bigg) \quad \forall \enspace k \in \{1,\ldots,K_t\}, \enspace \forall \enspace t \in {\mathbb N},
\end{displaymath}
to refer to the $k$th histogram bin at iteration $t$ (the index $t$ is needed to uniquely index a bin since the number of bins changes as $t$ increases).
The number of events in bin $B_{k,t}$ in a single observation of the Poisson process is a Poisson random variable with parameter $\int_{B_{k,t}}\lambda(x)\,{\rm d}x$.
Since this depends on the width of the bin, we instead estimate the average rate function in a bin, defined as
\begin{displaymath}
\psi_{k,t}=K_t\int_{B_{k,t}}\lambda(x)\,{\rm d}x.
\end{displaymath} 
We place independent truncated Gamma (TG) priors on each of the $\psi_{k,t}$ parameters, with shape and scale parameters $\alpha$ and $\beta$ and support on $[0,\lambda_{\max}]$ where $\lambda_{\max}$ is some known upper bound on the maximum of rate functions. (The TG$(\alpha,\beta,0,\lambda_{\max})$ distribution has a density proportional to a Gamma$(\alpha,\beta)$ distribution, but with truncated support $[0,\lambda_{\max}]$.)  In practice the $\lambda_{\max}$ parameter may be chosen very conservatively; setting $\lambda_{\max}$ to be too large does not affect the action selection; however it is important to include an upper limit on the prior support to permit tractable regret analysis, and the chosen $\lambda_{\max}$ appears in the regret bound in Theorem \ref{thm::regretupper}.

The consequence of this formulation is that, conditional on actions and observations in the first $t$ rounds, we have a posterior distribution over $\lambda$ at time $t$ which is piecewise constant. A $\lambda_t$ sampled from this posterior takes the form
\begin{align}
    \lambda_t(x) &= \sum_{k=1}^{K_t} \mathbb{I}\{x \in B_{k,t}\} \tilde{\psi}_{k,t},\quad\mbox{with}\nonumber\\
    \tilde{\psi}_{k,t}
    &\sim {\rm TG}(\alpha + H_{k,t}(t), \beta + \Delta_t N_{k,t}(t),0,\lambda_{\max}),\label{eqn::TGdistn}
\end{align}
where $H_{k,t}(s) = \sum_{j=1}^s \sum_{l=1}^{m_j} \mathbb{I}\{B_{k,t} \subseteq A_j\}\mathbb{I}\{X_{j,l} \in B_{k,t}\}$ gives the number events observed up to iteration $s$ in bin $B_{k,t}$, and $N_{k,t}(s) = \sum_{j=1}^s \mathbb{I}\{B_{k,t} \subseteq A_j\}$ gives the number of times to iteration $s$ that bin $B_{k,t}$ has been sensed (see Section \ref{sec:Actions}).

\citet{Gugushvili2018} demonstrate that, with a full observation at each iteration, this posterior contracts to the truth at the optimal rate for any $h$-H\"{o}lder continuous rate function $\lambda$. In particular, \begin{displaymath}
\E(||\lambda_t-\lambda||_2) \leq t^{\frac{-2h}{2h+1}}
\end{displaymath}
if $N_{k,t}(t)=t$ for all $k \in [K_t]$ and $K_t = O(t^{1/3})$. We describe in the next sub-section how the same choice of $K_t$ gives favourable performance in our sequential decision problem, even when we only observe subintervals of $[0,1]$.

\subsection{Thompson sampling}\label{sec:Actions}

In order to make action selection feasible, and to facilitate the inference using histograms, we constrain the action set of the TS approach using the same (increasingly fine-meshed) grid that the inference is performed over.
In particular, in round $t$, the action $A_t$ is constrained to lie in the set of available actions $\mathcal{A}_t$, consisting of those intervals and unions of intervals where only entire bins (no fractions of bins) are covered and the action consists of at most $U$ subintervals. Recall $U$ is the number of sensors, and the restriction to at most $U$ intervals ensures that each sensor can be allocated a single contiguous subinterval. We allow the number of bins $K_t$ to increase at rate $O(t^{1/3})$ by doubling the number of bins in line with the growth of $t^{1/3}$.

\label{sec.thompson}
Our TS approach is described in Algorithm 1. In each round $t$, for each bin $k \in \{1,\ldots,K_t\}$, a rate $\tilde{\psi}_{k,t}$ is sampled according to (\ref{eqn::TGdistn}), and then an action is selected that would be optimal if the true rate function were the piecewise-constant combination of these rates. As each bin rate is sampled from the current posterior and the action selected is the optimal action for this set of sampled rates, the selected action is chosen according to the posterior probability that it is the optimal one available. The optimal action conditional on a given sampled rate can be determined efficiently and exactly using the approach described in Section \ref{sec.actionselection}.


\begin{algorithm}[htbp]
    \caption{Thompson Sampling}
    \label{alg::TS}
    \hrule
    \vspace{0.2cm}
    \textbf{Inputs:} Gamma prior parameters $\alpha, \beta >0$, upper truncation point $\lambda_{\max}$
    
    \textbf{Iterative Phase:} For $t\geq 1$
    \begin{itemize}
        \item For each $k \in \{1,\ldots,K_t\}$, evaluate $H_{k,t}(t-1)$ and $N_{k,t}(t-1)$ and sample an index \begin{displaymath}
        \tilde{\psi}_{k,t} \sim {\rm TG}(\alpha + H_{k,t}(t-1),\beta + \Delta_t N_{k,t}(t-1), 0,\lambda_{\max})
        \end{displaymath}
        \item Choose an action $A_t \in \mathcal{A}_t$ that maximises $r(A)$ conditional on the true rate being given by the sampled $\tilde{\psi}_{k,t}$ values, and observe the events in $A_t$
        \end{itemize}
        \hrule
        \vspace{0.2cm}
\end{algorithm}

\subsection{Action selection by iterative merging (AS-IM)}\label{sec.actionselection}

In this section we describe a routine, called action selection by iterative merging (AS-IM), for efficiently determining the optimal action conditional on a given sampled rate function.
For the piecewise constant $\lambda_t$ functions sampled by the TS approach, the above optimization problem can be formulated as an integer program in which each bin $B_{k,t}$ is either searched or not. \citet{Grant2018} solve this program (albeit for more general cost functions and fixed discretisation) using traditional integer programming methods, with exponentially high computation complexities in $K_t$ and $U$. We instead introduce an efficient optimal action selection policy with polynomial sample complexity.

Firstly, we introduce additional notation that will be useful for explaining the algorithm. Throughout this section we take $\lambda$ as fixed and piecewise constant on bins $B_{k,t}$, and provide a method to find $A^*$ for this $\lambda$. An action $A\in{\mathcal A}$ can be written as the union of disjoint intervals: $A=\cup_{u=1}^UI_{u}$ and $I_{u}\cap I_{u'}=\emptyset$ for all $1\le u, u'\le U$. Define the \emph{weight} of an interval $I \in [0,1]$ as $w(I)=\int_I(\lambda(x)-C) {\rm d}x$. Thus, we may write the optimal action as
\begin{eqnarray*}
A^* =\argmax_{\{I_{u}\}_{u=1}^U}\sum_{u=1}^U w(I_u).
\end{eqnarray*}



AS-IM creates an initial set of candidate intervals ${\mathcal{I}} = \{I_n\}_{n=1}^N$
such that each ${I}_n$ is the union of a number of adjacent $B_{k,t}$, and for $k=2,...,K_t$, $B_{k,t}$ and $B_{k-1,t}$ belong to the same $I_n$ if and only if $w(B_{k,t})$ and $w(B_{k-1,t})$ have 
the same sign. Notice that, by construction, the weights of adjacent intervals have opposite signs. If the number of intervals in $\mathcal{I}$ with positive weight is not bigger than $U$, AS-IM returns all such intervals as the optimal action. Otherwise, AS-IM proceeds to the next step. 


AS-IM iteratively reduces the number of intervals with positive weights by merging the intervals. Specifically, let $M= \{n\in\{2,\ldots,N-1\} \,:\, |w(I_n)| \leq |w(I_{n-1})|,\,|w(I_n)|\leq |w(I_{n+1})|\}$ be the set of intervals that should be considered for merging. If $M$ is empty, no further merging should take place. If $M$ is nonempty let $n=\argmin_M |w(I_n)|$ be the label in $M$ with the smallest absolute weight; AS-IM merges $I_n$ with its two neighbour intervals $I_{n+1}$ and $I_{n-1}$ into one interval and updates the set of intervals $\mathcal{I}$. The merging procedure repeats until either $M$ is empty or the number of intervals with positive weight equals $U$. At this point AS-IM returns the $U$ intervals with the largest weights as $I^*_1,I^*_2,...,I^*_U$. 

We have the following result on AS-IM guaranteeing its optimality and efficiency. The proof is given in the supplementary material via an induction argument.

\begin{theorem} \label{thm.ASIM}
The AS-IM policy returns the optimal action and its sample complexity is not bigger than $O(K_t\log K_t)$. 
\end{theorem}

\section{Regret Bound}
\label{sec.regret_bound}
In this section, we present our main theoretical contribution: an upper bound on the Bayesian regret of the TS approach. There is an inevitable minimum contribution to regret due to the optimal action likely not being in our discretised action set. But by allowing the mesh to become more fine as more observations are made, we will gradually reduce this discretisation regret and permit a closer approximation to the true underlying rate function.

For the analysis that follows it will be useful to define $A_t^*= \arg\max_{A \in \mathcal{A}_t} r(A)$ as the optimal action available in round $t$.
We then define for any $A \in \mathcal{A}_t$ and $t \in {\mathbb N}$:
\begin{align*}
    \delta(A) &= r(A^*)-r(A) \\
    \delta_t(A) &=r(A_t^*)-r(A)
\end{align*}
as the \emph{single-round regret} of the action $A$ with respect to the optimal continuous action and the optimal action available to the algorithm in round $t$ respectively.
The difference between $\delta(A)$ and $\delta_t(A)$ is that the ``discretisation regret'' by choosing actions only from ${\mathcal A}_t$ is present only in $\delta(A)$.
Minimising the true regret $\delta(A)$ requires balancing out estimation accuracy (requiring a coarse grid) versus discretisation regret (requiring a finer grid).
We find below that choosing the number of bins $K_t$ to be order $O(t^{1/3})$ provides the best theoretical performance guarantees. This coincides with the optimal posterior contraction rate findings in \citet{Gugushvili2018}. We verify this numerically in Section \ref{sec.sims} and find that this rebinning rate is superior to a faster linear rate of rebinning.


\begin{theorem} \label{thm::regretupper}
Consider the setup of Section \ref{sec.model}, with $U$ sensors, and cost of sensing $C$.
Suppose we choose $K_t$ such that there exist positive constants $\underline{K},\overline{K}$ such that $\underline{K}t^{1/3}\leq K_t\leq \overline{K}t^{1/3}$.
Then the Bayesian regret of Algorithm \ref{alg::TS} satisfies
\begin{align*}
    BReg&(T) \leq 4\overline{K}\big(\log(T+1)\log(T) + 2\lambda_{\max}\big)T^{1/3} \nonumber \\
            &\quad + \big(CU\underline{K}^{-1} + \sqrt{24\overline{K}\lambda_{\max}\log(T)}\big)T^{2/3}.
\end{align*}
\end{theorem}

This main result is that we have a $O(T^{2/3}\log^{1/2}(T))$ bound on the Bayesian regret. A lower bound for the problem is not currently available. The closest result available is that of \citet{Kleinberg2005} for CABs with bounded Lipschitz smooth reward function and bounded noise. The bound holds only for a one-dimensional action space and is of order $\Omega(T^{2/3})$. The material differences in our setting are that the observation noise is unbounded (with Poisson tails), our reward function is defined on higher dimension (the unrestricted action space of the underlying CAB is of dimension $2U$), and that we observe additional information in the form of event locations. In the context of the nearest related results therefore, Theorem 2 suggests that the TS approach is a strongly performing policy. 

\begin{proof}[Proof of Theorem \ref{thm::regretupper}]
The Bayesian regret can be decomposed as the sum of the regret due to discretisation and the regret due to selecting suboptimal actions in $\mathcal{A}_t$, as follows \begin{displaymath}
BReg(T) = \E\bigg(\sum_{t=1}^T \delta(A^*_t)\bigg) + \E\bigg(\sum_{t=1}^T \delta_t(A_t) \bigg)
\end{displaymath}

The expectation in the first term only averages over $\lambda$ functions, not over action selection, and the sum can be upper bounded uniformly over all $\lambda$'s by considering the rate of re-binning. In particular we have the following lemma, proved in the supplementary material.
\begin{lemma}
The regret due to discretisation is bounded by \begin{equation*}
    \sum_{t=1}^T\delta(A_t^*) \leq CU\underline{K}^{-1}T^{2/3},
\end{equation*}
uniformly over all rates $\lambda$.
\end{lemma}

To handle the stochastic part of the regret we use a decomposition from Proposition 1 of \citet{RussoVanRoy2014}. For all $T$, for all $1\leq t\leq T$ and for all $A\in {\mathcal{A}}_t$, let $L_{t,T}(A)$ and $U_{t,T}(A)$ satisfy
$-C|A|\leq L_{t,T}(A)\leq U_{t,T}(A)$ (see below for a judicious choice of these variables).
Then, for any $T$,
\begin{align*}
    &\enspace \E\left[\sum_{t=1}^T \delta_t(A_t)\right] 
    = \E\!\left[\sum_{t=1}^T r(A^*_t)-r(A_t) \right] \\
    &=\E\!\left[\sum_{t=1}^T U_{t,T}(A_t)\!-\!r(A_t) \right] + \E\!\left[\sum_{t=1}^T r(A^*_t)\!-\!U_{t,T}(A^*_t)\right] \\
    &\leq \E\left[\sum_{t=1}^T U_{t,T}(A_t) \!-\! L_{t,T}(A_t)\right] +\lambda_{\max}\times\\
    &
    \,\left[\sum_{t=1}^T P\!\left(r(A^*_t) \!>\! U_{t,T}(A^*_t)\right) 
    + \sum_{t=1}^T P\!\left(r(A_t) \!<\! L_{t,T}(A_t)\right)\right]
\end{align*}
The key step here is the second equality, which holds for TS because the distribution of $U_t(A_t)$ is precisely the distribution of $U_t(A^*_t)$ due to the method of selecting $A_t$.
The final step follows by noting that, for any $A$, 
\begin{align*}
\E&[r(A)-U_{t,T}(A)]\\
&\leq \E\left[(r(A)-U_{t,T}(A)){\mathbb I}_{\{r(A)-U_{t,T}(A)>0\}}\right]\\&
\leq \lambda_{\max} P\left(r(A) > U_{t,T}(A)\right),
\end{align*}
and similarly for $E[L_{t,T}(A)-r_t(A)]$. The $\lambda_{\max}$ term arises from $r(A) \leq \lambda_{\max}-C|A|$ and $U_{t,T}(A) \geq -C|A|$ for all $A \in \mathcal{A}_t$.

We will choose $L_{t,T}$ and $U_{t,T}$ so that each sum converges. In particular, the confidence bounds derived in \citet{Grant2018} for Poisson random variables inspire the definition of
\begin{displaymath}
D_{k,T}(t-1) = \frac{2\log(t)}{\Delta_TN_{k,T}(t-1)} + \sqrt{\frac{6\lambda_{\max}\log(t)}{\Delta_TN_{k,T}(t-1)}}
\end{displaymath}
for all $k \in [K_T]$, 
with upper and lower confidence bounds on the reward of an action $A \in \mathcal{A}_t$ at time $t \in \mathds{N}$ as follows: \begin{align*}
    U_{t,T}(A) &= \Delta_T\!\!\sum_{k:B_{k,T}\subseteq A}\!\!\!\hat{\psi}_{k,T}(t\!-\!1) + D_{k,T}(t\!-\!1) - C|A|, \\
    L_{t,T}(A) &= \Delta_T\!\!\sum_{k:B_{k,T}\subseteq A}\!\!\!\hat{\psi}_{k,T}(t\!-\!1) - D_{k,T}(t\!-\!1) - C|A|,
\end{align*} 
where $\hat{\psi}_{k,T}(t)= \frac{H_{k,T}(t)}{\Delta_TN_{k,T}(t)}$
gives the empirical mean in bin $B_{k,T}$ after $t$ rounds. It is in the definition of $U_{t,T}$ and $L_{t,T}$ that we see the need for a $T$-dependence in our choice of upper and lower confidence bounds---we need to count the number times actions $A_t$ for $t<T$ selected the bin $B_{k,T}$ defined for time $T$.

In the supplementary material we prove the following lemmas, which when combined are sufficient to complete the proof of Theorem \ref{thm::regretupper}. \begin{lemma}
For $U_{t,T}$ and $L_{t,T}$ as defined above, we have
\begin{align*}
\sum_{t=1}^T& U_{t,T}(A_t) - L_{t,T}(A_t) \leq \\
    &4 \overline{K} \log(T)\log(T+1)T^{1/3} + \sqrt{24\overline{K}\lambda_{\max}\log(T)}T^{2/3}
\end{align*}
\end{lemma}

\begin{lemma}
The deviation probabilities can be bounded \begin{equation*}
    P\bigg(r(A_t) \notin [L_{t,T}(A_t),U_{t,T}(A_t)]\bigg) \leq 2K_Tt^{-2}
\end{equation*}
\end{lemma}

Combining these results we have: \begin{align*}
    BReg&(T) \leq CU\underline{K}^{-1}T^{2/3} + 4\overline{K}\log(T)\log(T+1)T^{1/3} \\
            & 
            + \sqrt{24\overline{K}\lambda_{\max}\log(T)}T^{2/3} 
           + 2K_T \lambda_{\max}\sum_{t=1}^T 2t^{-2} 
\end{align*}
which gives the required result as $\sum_{t=1}^\infty t^{-2} \leq \frac{\pi^2}{6}$.
\end{proof}

\section{Simulations}
\label{sec.sims}
%
In this section, we provide simulation examples on the performance of the Thompson sampling approach presented in Section~\ref{sec.thompson}. We first examine the effect of the rebinning rate on the regret and then investigate the performance of the Thompson sampling approach in relation to other algorithms. 

\begin{figure}[htbp]
\centering
\includegraphics[width=0.75\columnwidth]{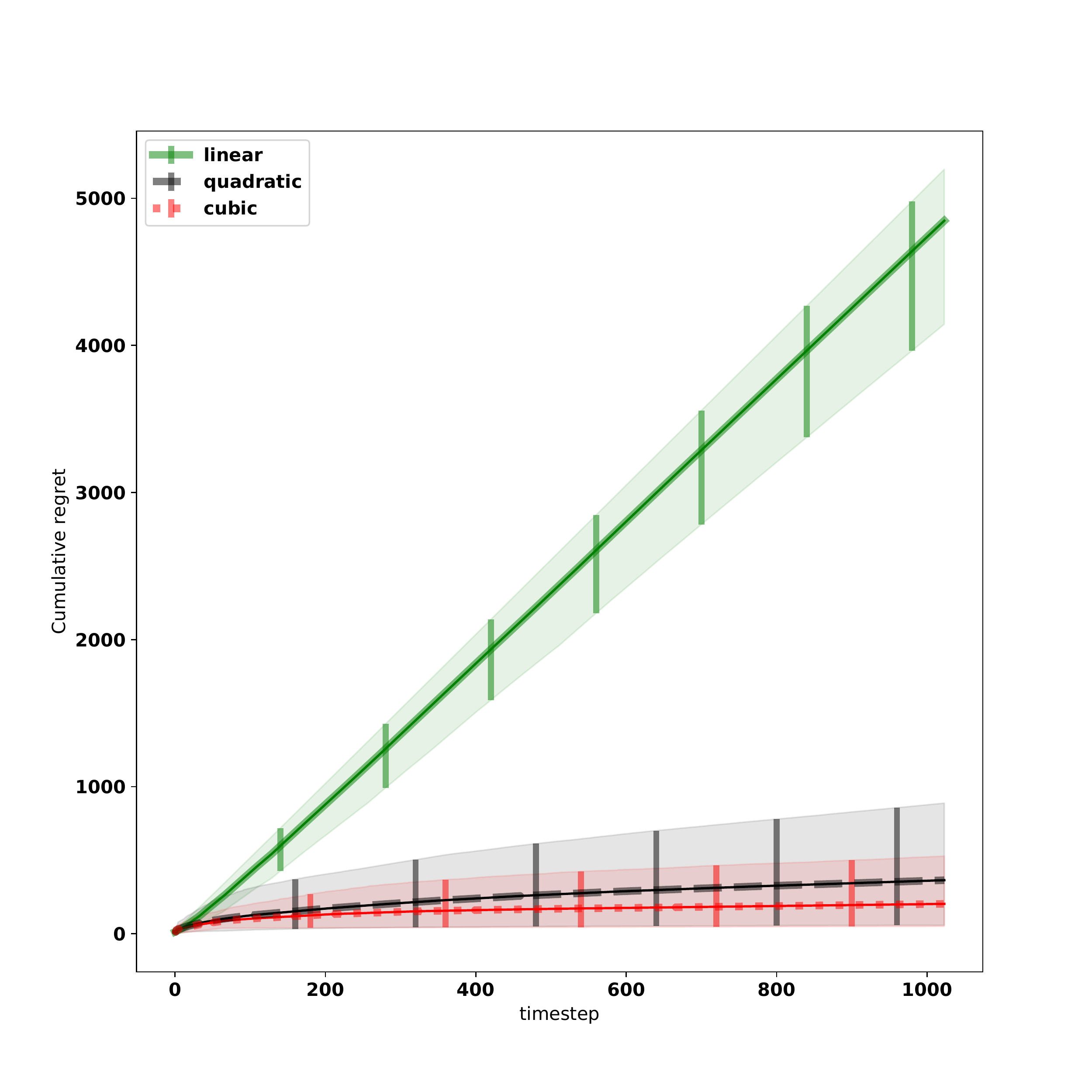}
 \caption{Cumulative regret comparing different rebinning rates.} \label{fig.syn.rebinrates_regret}
\end{figure}

\subsection{Effect of rebinning rate}\label{sec::rebinning}
Firstly we examine the effect of different rebinning rates in a simple unimodal setting with  $\lambda(x)=\frac{1000}{21}(x-x^2)$, $C=10$, and $U=1$ sensor. This setting is chosen such that the optimal action can be calculated as $A^* = [0.3, 0.7]$. Here, and throughout our experiments, we set the prior parameters for Thompson sampling to be $\alpha=0.5$ and $\beta=0.5/C$, where scaling by cost $C$ makes the prior relevant to the expected scale of costs in the problem. We also set the truncation $\lambda_{\max}$ to be ten times the true maximal value of $\lambda$; $\lambda_{\max}$ is an inconvenient parameter that is only needed for the theory, so we set it to a conservative large value that should have no influence on the real behaviour of the algorithm. The experiment is run 10 times for $T=1024$ timesteps starting with $K_0=4$ bins.

We compare linear, square root and cube root rebinning rates: the number of bins $K_t$ is doubled in rounds where $t$ (in the linear case), $t^{1/2}$ (square root case) or $t^{1/3}$ (cube root case) is twice its value at the last rebinning time. Actions are selected using the TS method of Algorithm \ref{alg::TS} and Fig.~\ref{fig.syn.rebinrates_regret} shows that the cumulative regret is consistently lower under the cube root rate. While under the linear rebinning rate, actions with reward close to that of $A^*$ become available more quickly, reducing the discretisation regret, the issue is that the majority of bins contain very little data and the posterior inference is heavily dependent on the prior. Under the cube root (and indeed square root) rebinning rate the action set grows more slowly but the unavoidable discretisation regret is balanced by better action selection. The square root case is surprisingly similar to the cube root case despite a weaker theoretical rate in this case. We demonstrate the shrinking of the discretisation regret in the supplementary material.

We also show, in Fig.~\ref{fig.syn.rebinrates.posterior}, the posterior inference under the linear and cube root settings at the last time step of one run of the experiment. 
The posterior under the linear rebinning is highly unconcentrated with simply insufficient numbers of observations in almost all bins.  
The cube root rate on the other hand results in a posterior which is much more concentrated about the truth in the region where it matters.

\begin{figure}[htbp]
\centering
\subfigure[Linear]{\includegraphics[width=0.7\columnwidth]{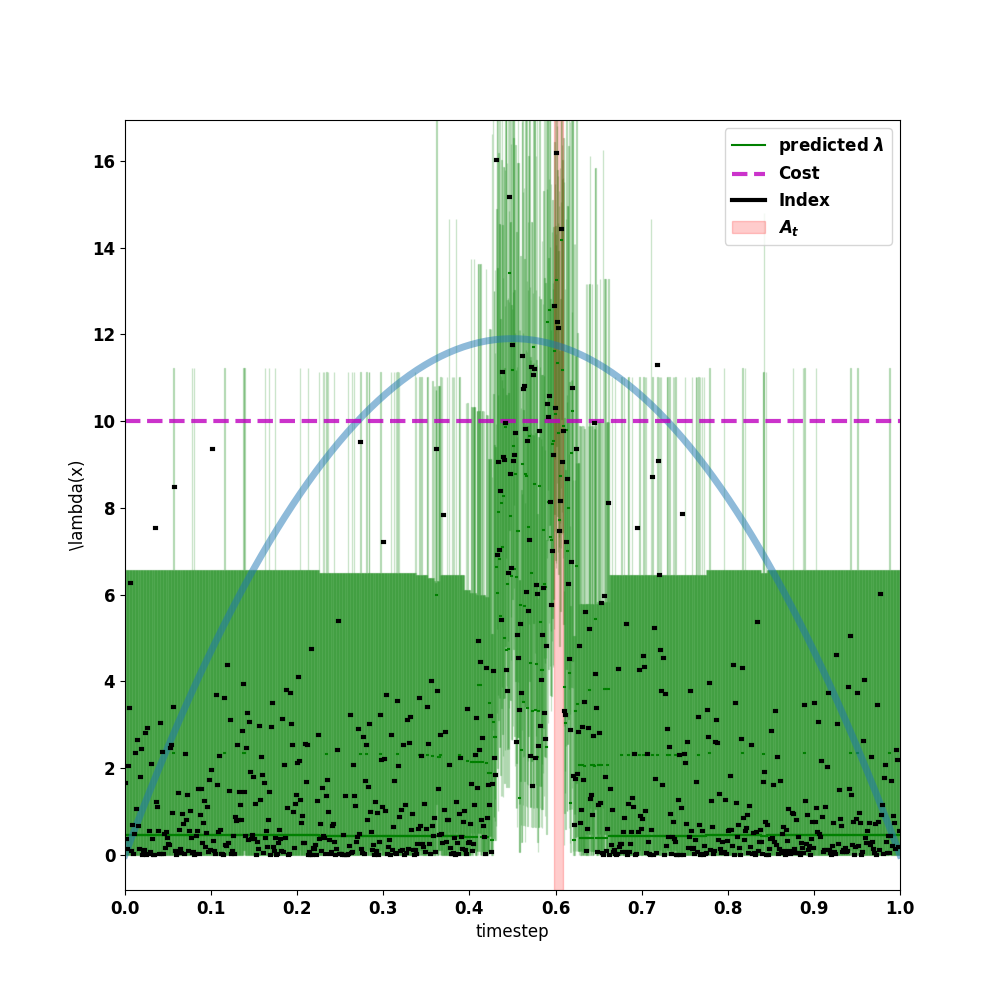}}

\subfigure[Cube Root] {\includegraphics[width=0.7\columnwidth]{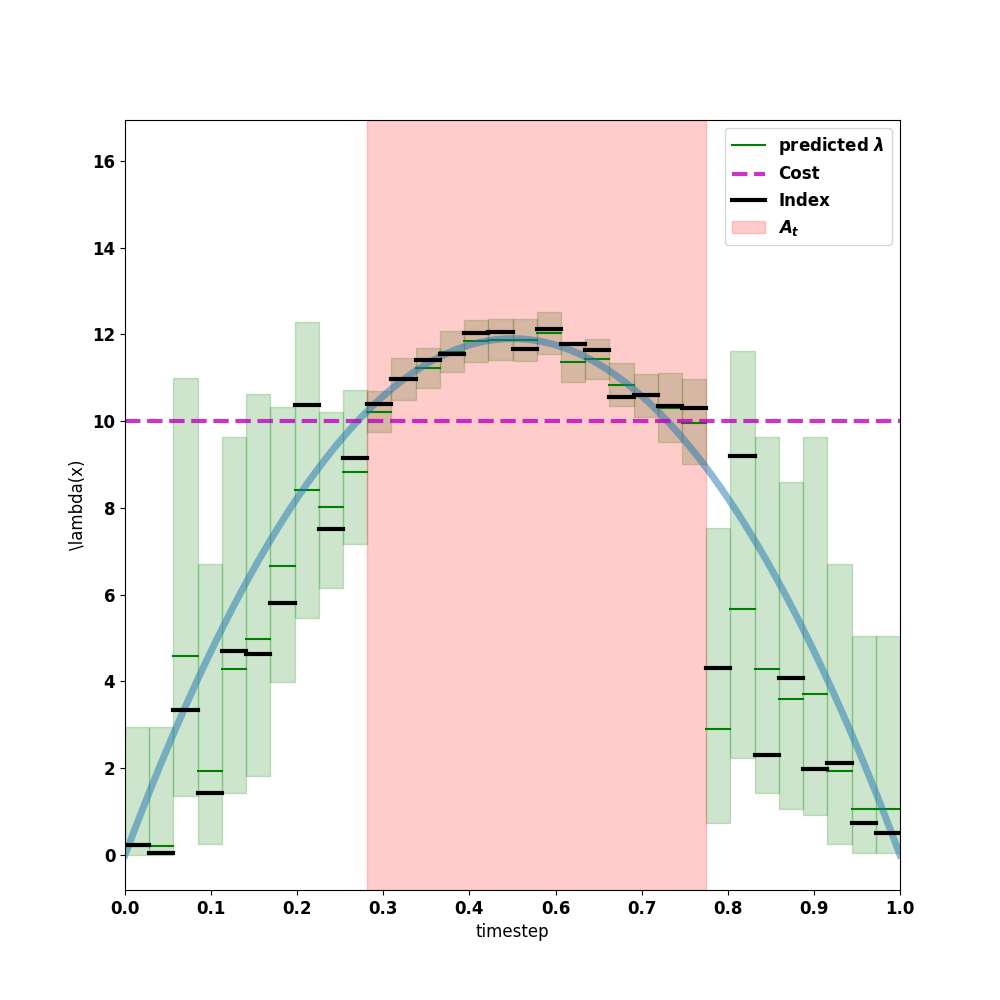}}
  \caption{Posterior under the linear and cube root rebinning rates at round $T=1024$. We show the true rate function (blue) and cost (pink), the posterior credible interval (light green) and mean (dark green) per bin. Thompson samples are shown in black, and the selected interval, $A_T$, is the (red) vertical bar. The initial number of bins is 4 in both cases and the final number of bins, $K_T$, is 2048 for the linear rebinning schedule and 32 bins for the cube root schedule.}\label{fig.syn.rebinrates.posterior}
\end{figure}

\subsection{Comparison to Baselines}
We now compare different baseline policies solely using the cube root rebinning schedule. Experiments with the unimodal rate of Section \ref{sec::rebinning} were not informative since the problem is an easy one. We instead use a bimodal rate $\lambda(x) =\max\big(0.001,\frac{15 \sin(10x)}{
                                                       \sqrt{(10x + 1)} + x}\big)$ with $C=2$ and $U=2$ sensors.
Each experiment was run 10 times for $T=1000$ time steps, starting with $K_0=16$ bins and terminating with $K_T=128$ bins.
In addition to the {Thompson sampling} approach described in Section~\ref{sec.thompson}, we consider three other algorithms, which are summarised here and described precisely in the supplementary material.
(i) An upper confidence bound (UCB) approach, in which the decision-maker chooses what would be an optimal action if the true rates were $U_{t,t}$ (as defined in the proof of Theorem 1); this is essentially the FP-CUCB algorithm of \citet{Grant2018}, albeit with a changing mesh, and requires the specification of an upper bound $\lambda_{\max}$ on the rate in order to define the action selection. In our experiments we fix this $\lambda_{\max}$ to the correct value; in practise a conservative estimate is usually available, but for this algorithm the choice of $\lambda_{\max}$ strongly affects the actions selected, in contrast with the TS algorithm, and we choose the most favourable $\lambda_{\max}$ for this algorithm.
(ii) A modified-UCB approach ({mUCB}) where the empirical mean for each histogram bin $\hat{\psi}_{k}$ is used in place of the overall maximum rate $\lambda_{\max}$. Note this modification invalidates the concentration results used in \citet{Grant2018}, but appears to improve performance in practice.
(iii) An {$\epsilon$-Greedy} approach where the intervals are selected according to the empirical mean for each bin $\hat{\psi}_{k}$ but occasionally an explorative randomisation step occurs in which the algorithm samples, for each bin, a draw from the prior. The randomisation step is taken with probability $\epsilon=0.01$.

\begin{figure} [htbp]
\centering
\includegraphics[width=0.75\columnwidth]{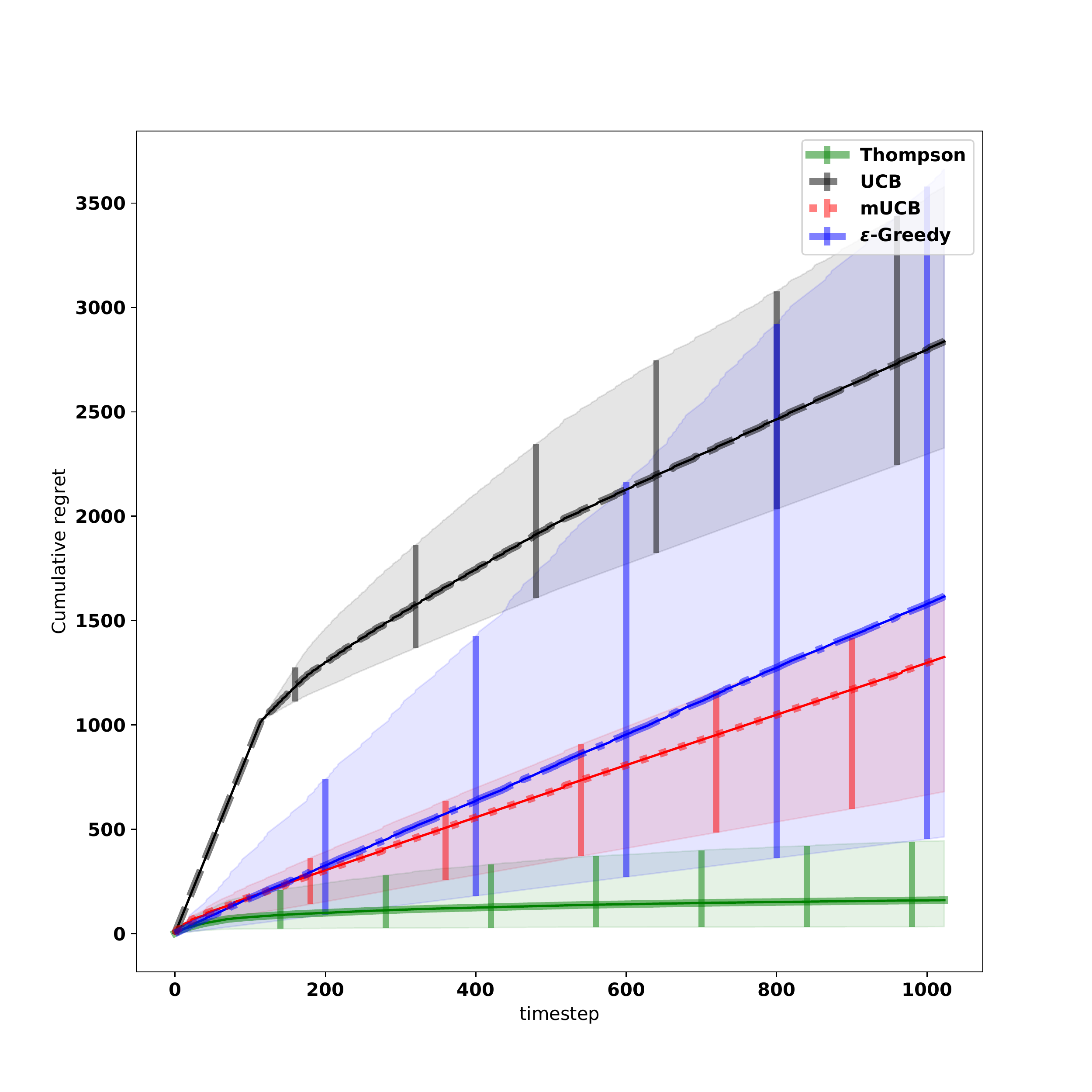}
\caption{Cumulative regret plot for the bimodal rate functions. The experiments are repeated 10 times and the mean and 95\% empirical confidence interval is shown for each policy.}\label{fig.syn.regret_plots}
\end{figure}

The cumulative regret for each policy is shown in Figure~\ref{fig.syn.regret_plots}. The worst performing policy is the UCB approach, despite its theoretical properties.
The poor performance of the UCB policy is due to the overestimation of the true rate as can be seen in the illustrative example shown in Figure~\ref{fig.syn.multimodal}(d). Even after 900 iterations, the UCB values (in black) are close to the cost threshold even in the regions where the true rate is low and there is little uncertainty. In contrast the modified-UCB values, that do not depend on $\lambda_{\max}$, are less inflated where the uncertainty is low (Figure~\ref{fig.syn.multimodal}(c)) resulting in more often choosing a better action.
In Fig.\ \ref{fig.syn.regret_plots} the $\epsilon$-Greedy achieves similar mean regret to modified-UCB but with a higher variance.
The $\epsilon$-Greedy approach has the highest variance due to the greediness of the algorithm. A higher value of $\epsilon$ would reduce variance but would increase the exploration cost. The TS approach consistently outperforms all other policies.


Further intuition can also be gained from the posterior examples shown in Figure~\ref{fig.syn.multimodal}. 
These were selected at time step $T=900$ from one of the experimental runs. 
The TS approach has selected an action close to optimal. Further, the posterior variance outside the optimal interval is significantly higher that in the selected regions as only a small number of observations were taken in those regions demonstrating the high efficiency of the method. In contrast both UCB approaches have uniformly low posterior variance in the entirety of the domain reflecting the large number of observations taken incurring a high exploration cost.
In contrast, the $\epsilon$-Greedy approach selects smaller than optimal intervals with high posterior variance outside these regions. This reflects an under-exploration of the greedy approach which is only able to escape bad local minima when the randomisation step is used.

\begin{figure}[htb]
\centering
\mbox {
\subfigure[TS, $A_t=$  \newline {$[0.007, 0.281],
 [0.687, 0.882]$}]{\includegraphics[width=0.5\columnwidth]{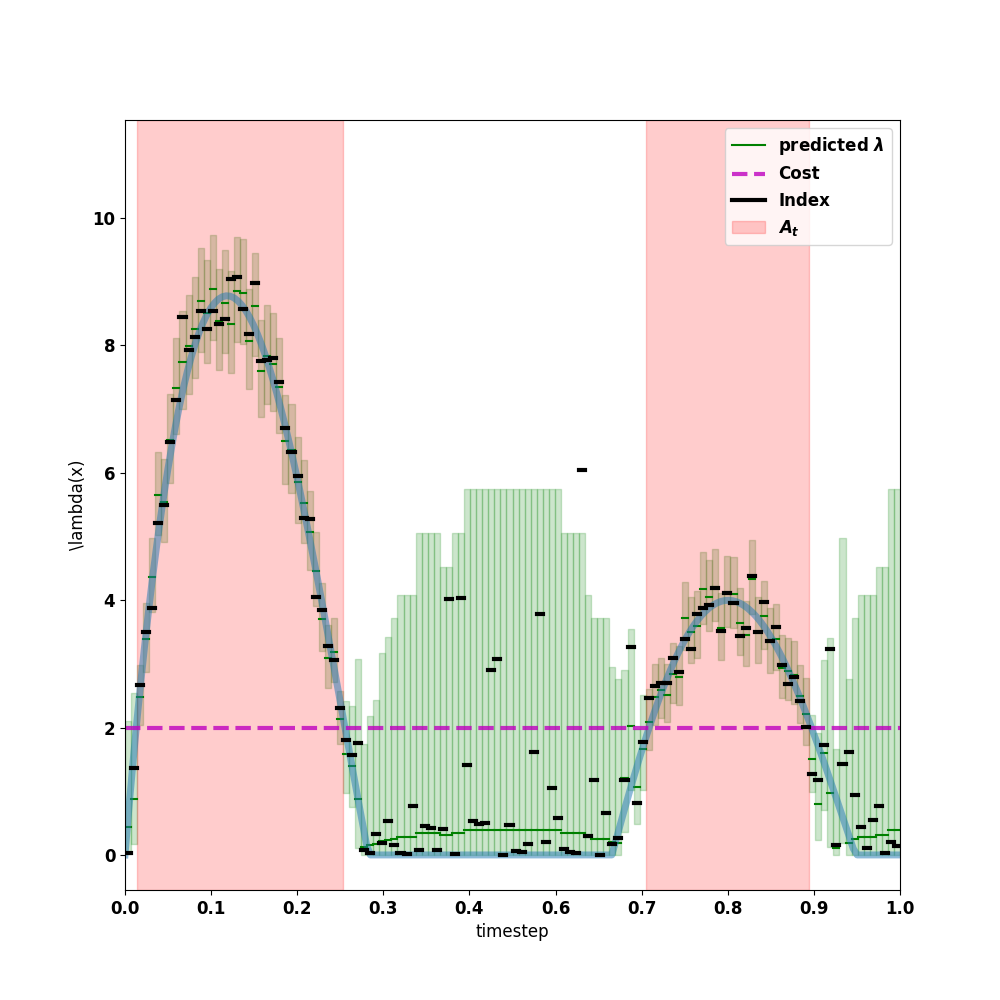}}
\subfigure[$\epsilon$-greedy, $A_t=$   \newline  {$[0.062,  0.312 ], 
 [0.718, 0.812 ]$}]{\includegraphics[width=0.5\columnwidth]{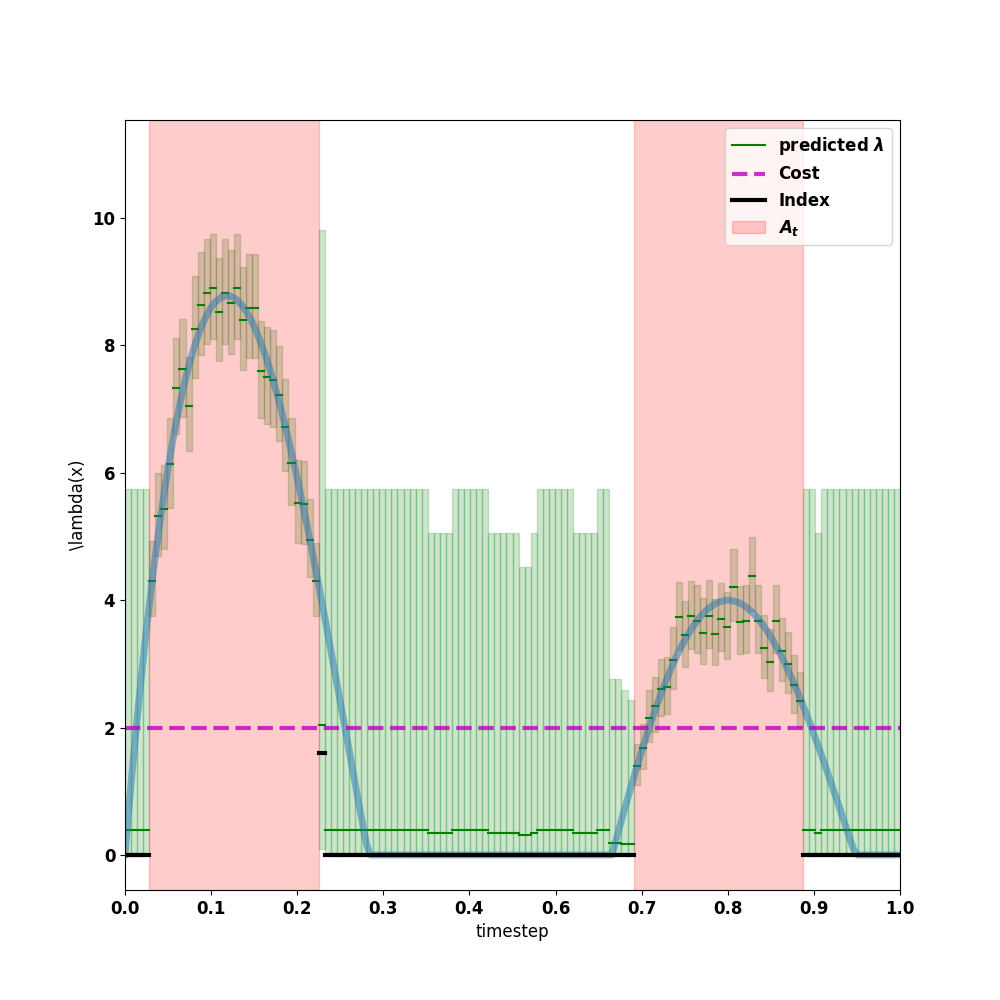}}
} \mbox {
\subfigure[mUCB, $A_t=$ \newline   {$[0., 0.320], 
 [0.640, 0.937]$}]{\includegraphics[width=0.5\columnwidth]{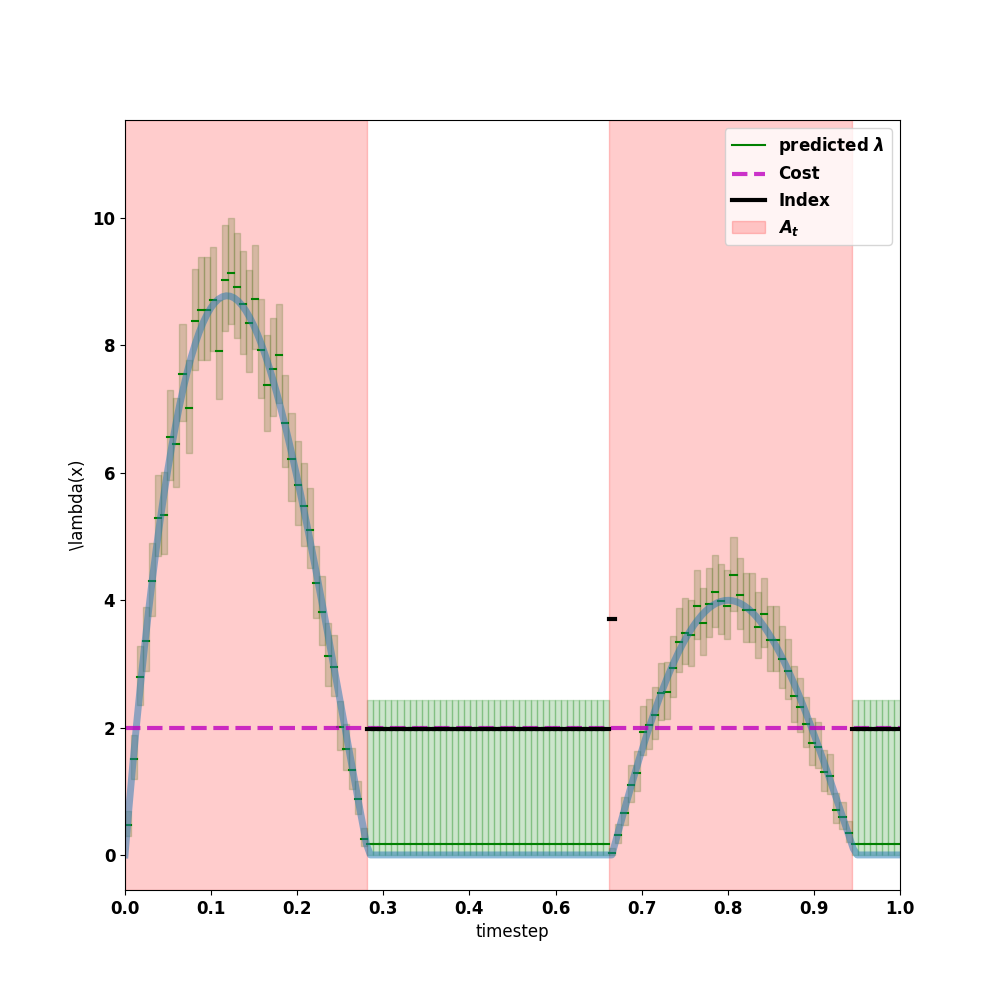}}
\subfigure[UCB, $A_t=$  \newline {$[0., 0.320],
 [0.640, 0.937]$} ] {\includegraphics[width=0.5\columnwidth]{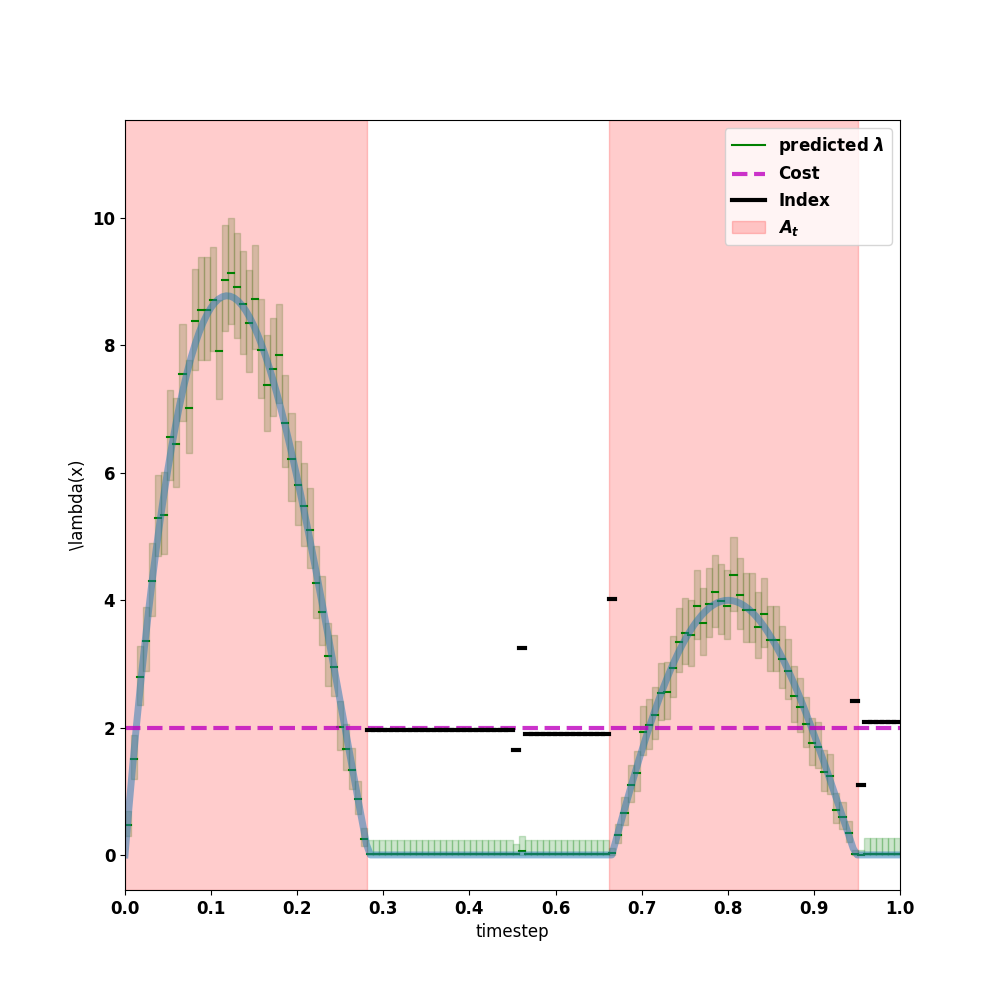}}
}
  \caption{Posterior under different action selection strategies for the bimodal test function. The true rate function (orange), posterior mean (blue) and 95\% confidence interval (green infill) is shown. Rate samples for each method are shown in black for each bin and the cost threshold is the (magenta) horizontal dashed line. The optimal action is to select two intervals $A^* = [0.013, 0.280], [0.675, 0.882]$. }\label{fig.syn.multimodal}
\end{figure}

In summary, the TS approach outperforms all the other approaches we have considered and is able to efficiently trade-off exploration penalty and exploitation reward.



\section{Conclusion}
We have presented a continuum-armed bandit model of sequential sensor placement. This model introduces the complexities of point process data and heavy-tailed reward distributions to continuum-armed bandits for the first time through its Poisson process observations. We proposed a Thompson sampling approach to make decisions based on fast non-parametric Bayesian inference and an increasingly granular action set, and derived an upper bound on the Bayesian regret of the policy which is independent of the choice of prior distribution. 

In our simulation study we have studied two aspects of our approach. Firstly we examined the effect of the rebinning rate on posterior inference and regret. The theoretically-optimal cube root rate resulted in more accurate posterior inference than a linear or square root rebinning rate. This effect was also evident in a lower regret for the cube root rate. 

Our empirical study also contrasted our Thompson sampling approach to alternative approaches like UCB or $\epsilon$-greedy policies. In both the cases we examined, we found the other methods either over-explored (e.g.\ UCB) or over-exploited (e.g.\ $\epsilon$-greedy). The TS approach achieved the best trade-off between the two and consistently achieved the lowest regret.

The observation model and rebinning strategies we have presented here are straightforward; it would be interesting to extend the algorithm and analysis to account for imperfect observations and to allow for heterogeneous bin widths, letting us capture more detail of the rate function in areas where we have made many observations and adopt a smoother estimate in others. 

An alternative to the discretisation approach we have followed is to employ a continuous model such as a Cox process for which efficient approximate inference methods exist~\cite{John2018}. Action selection under the additive cost model would still be possible via a continuous action space extension of the AS-IM routine. The regret analysis in this setting would be more involved although recent concentration results \citep[e.g.][]{Kirichenko2015} suggest possible approaches.

\clearpage

\bibliography{SensorPlacementRefs}
\bibliographystyle{icml2019}

\appendix
\onecolumn

\section{Regret bound proofs}
\subsubsection*{Proof of Lemma 1}
Define $A_{\min,t} =
\bigcap_{A\in{\mathcal A}_t \,:\, A^*\subseteq A} A$ as the smallest interval (or union of intervals) in $\mathcal{A}_t$ containing the optimal interval (or union of intervals).
It will be easier to bound the regret of $A_{\min,t}$ than $A^*_t$ wrt $A^*$.
We have, for $t \in {\mathbb N}$, 
\begin{align*}
     \delta(A^*_t) &=  r(A^*)- r(A^*_t) \\
                               &\leq  r(A^*) - r(A_{\min,t}) \\
                               &=  \int_{A^*}(\lambda(x)-C)\,{\rm d}x - \int_{A_{\min,t}}(\lambda(x) - C)\,{\rm d}x \\
                               &=  C|A_{\min,t}\setminus A^*|-\int_{A_{min,t}\setminus A^*} \lambda(x)dx   \\
                               &\leq 2CU\Delta_t.
                               \end{align*}
Here, the final inequality holds since $2\Delta_t$ bounds the difference between the lengths of subintervals of $A_{min,t}$ and $A^*_t$, and there are $U$ such subintervals. Since $\Delta_t=K_t^{-1}\leq\underline{K}^{-1}T^{-1/3}$ the result follows immediately.

\subsubsection*{Proof of Lemma 2}

Consider the term inside the expectation \begin{align*}
   \sum_{t=1}^T U_{t,T}(A_t) - L_{t,T}(A_t)    &= 2\Delta_T \sum_{t=1}^T \sum_{k: B_{k,T} \subseteq A_t} D_{k,T}(t-1) \\
    &= 2\Delta_T \sum_{t=1}^T \sum_{k: B_{k,T} \subseteq A_t} \frac{2\log(t)}{\Delta_TN_{k,T}(t-1)} + \sqrt{\frac{6\lambda_{\max}\log(t)}{\Delta_T N_{k,T}(t-1)}} \\
    &= 2\Delta_T \sum_{t=1}^T \sum_{k=1}^{K_T} \mathbb{I}\{B_{k,T} \subseteq A_t\}\bigg(\frac{2\log(t)}{\Delta_T\sum_{s=1}^{t-1}\mathbb{I}\{B_{k,T} \subseteq A_s\}} + \sqrt{\frac{6\lambda_{\max}\log(t)}{\Delta_T{\color{black}\sum_{s=1}^{t-1}}\mathbb{I}\{B_{k,T} \subseteq A_s\}}}\bigg)  \\
    &\leq 2\Delta_T \sum_{k=1}^{K_T} \sum_{j=1}^{N_{k,T}} \frac{2\log(
    {\color{black} T})}{j\Delta_T} + \sqrt{\frac{6\lambda_{\max}\log(
    {\color{black} T})}{j\Delta_T}} \\
    &\leq 2\Delta_T K_T\bigg(\sum_{j=1}^T \frac{
    {\color{black} 2}\log(
    {\color{black} T})}{j\Delta_T} + \sum_{j=1}^T \sqrt{\frac{6\lambda_{\max}\log(
    {\color{black} T})}{j\Delta_T }} \bigg) \\
    &= 4K_T\log(T)\log(T+1) + \sqrt{24\lambda_{\max}K_T\log(T)}T^{1/2} \\
    &{\color{black}\leq 4 \overline{K} \log(T)\log(T+1)T^{1/3} + \sqrt{24\overline{K}\lambda_{\max}\log(T)}T^{2/3}}
\end{align*}
{\color{black} where the penultimate line is due to $\Delta_T=K_T^{-1}$, and the final inequality is because $K_T\leq \overline{K} T^{1/3}$.}

\subsubsection*{Proof of Lemma 3}
We have the following, which holds for any round $t$ \begin{align*}
    &\enspace P\bigg(r(A_t) \notin [L_{t,T}(A_t),U_{t,T}(A_t)] \bigg) \\
    &\leq P\bigg(r(A_t) \leq L_{t,T}(A_t)\bigg) + P\bigg(r(A_t) \geq U_{t,T}(A_t)\bigg) \\
    &= P\bigg(\sum_{k:B_{k,T} \subseteq A_t} \psi_{k,T} \leq \sum_{k:B_{k,T} \subseteq A_t} \left[\hat{\psi}_{k,T}(t-1) - D_{k,T}(t-1)\right] \bigg) \\
    &\quad + P\bigg(\sum_{k:B_{k,T} \subseteq A_t} \psi_{k,T} \geq \sum_{k:B_{k,T} \subseteq A_t} \left[\hat{\psi}_{k,T}(t-1) + D_{k,T}(t-1)\right]\bigg) \\
    &\leq \sum_{k:B_{k,T}\subseteq A_t} \Bigg[P\bigg(\psi_{k,T} -\hat{\psi}_{k,T}(t-1) \leq -D_{k,T}(t-1)\bigg) + P\bigg(\psi_{k,T} - \hat{\psi}_{k,T}(t-1) \geq D_{k,T}(t-1)\bigg)\Bigg] \\
    &\leq \sum_{k=1}^{K_T}P\bigg(|\psi_{k,T}-\hat{\psi}_{k,T}(t-1)|\geq \frac{2\log(t)}{\Delta_TN_{k,T}(t-1)}+\sqrt{\frac{6\lambda_{\max}\log(t)}{\Delta_TN_{k,T}(t-1)}} \bigg) \\
    &\leq \sum_{k=1}^{K_T}\sum_{s=1}^{t-1} P\bigg(|\psi_{k,T}-\hat{\psi}_{k,T}(t-1)|\geq \frac{2\log(t)}{\Delta_TN_{k,T}(t-1)}+\sqrt{\frac{6\lambda_{\max}\log(t)}{\Delta_TN_{k,T}(t-1)}} \enspace \bigg| \enspace N_{k,T}(t-1)=s\bigg)\leq 2K_T t^{-2}.
\end{align*}
The final inequality is a direct application of Lemma 1 of \cite{Grant2018} which in turn exploits Bernstein's Inequality for independent Poisson random variables.

\section{Proof of optimality and efficiency of AS-IM}
\subsubsection*{Proof of Theorem 1}

Recall that the reward of an action is the sum of the weights of the intervals that comprise that action.

We prove the theorem by induction. Assume at least one initial $I_n$ has a positive weight (otherwise the optimal action is to do no sensing). For $N=1$ initial interval, which therefore has a positive weight, AS-IM simply returns this interval, which is optimal. For $N=2$ initial intervals, with one positive weight, AS-IM returns the postitively-weighted interval, which is the optimal action. Now, assuming AS-IM returns the optimal action for $N\ge 1$, we prove that AS-IM returns the optimal action for $N+2$ initial intervals. The result follows by induction.

Given ${\mathcal{I}}=\{{I}_n\}_{n=1}^{N+2}$, if the number of intervals in $\mathcal{I}$ with positive weight is not bigger than $U$, AS-IM returns all such intervals. This is the optimal action since all bins with positive reward can be covered without incurring the cost of any bins with negative reward; any other action either omits a positive-reward bin, or includes a negative-reward bin.

Similarly, consider the situation in which no interval satisfies the merging condition. Suppose that the optimal action $A^*$ places a sensor on a sequence of intervals $I_{m}\cup\cdots\cup I_n$ with $n>m$. Clearly we must have $w(I_m)>0$ and $w(I_n)>0$ since otherwise the total weight could be increased by omitting the negatively-weighted end interval. But the fact that no interval can be merged implies that either $|w(I_{m+1})|>|w(I_m)|$ or $|w(I_{n-1})|>|w(I_n)|$. Hence removing either $I_m\cup I_{m+1}$ or $I_{n-1}\cup I_n$ from the sensor will improve the total weight. It follows that, under $A^*$, each sensor is allocated to a single interval, and allocating to the $U$ highest-weight intervals, as specified by AS-IM, maximises the reward.

Now, assume that at least one interval is merged in AS-IM. Let $I_n$ be the interval which minimises $|w(I_n)|$ and so is the first interval which is merged with its neighbours in AS-IM into a single interval $\tilde{I}_n = I_{n-1} \cup I_n \cup I_{n+1}$. Let $\tilde{A}^*$ be AS-IM's solution for the set of intervals $\tilde{\mathcal{I}}=\{I_1,\cdots, I_{n-2}, \tilde{I}_n, I_{n+2},\cdots, I_{N+2}\}$. By induction, $\tilde{A}^*$ is optimal for $\tilde{I}$. We prove that $A^*$, the optimal solution for $\mathcal{I}$, is equal to $\tilde{A}^*$. To prove this, we consider different cases based on the sign of $w(I_n)$.

\paragraph{Case 1: $w(I_n)<0$.} First note that the optimal solution cannot include only one neighbour of $I_n$. If $I_{n-1}$ were included but $I_{n+1}$ were not, we could add both $I_n$ and $I_{n+1}$ and increase the overall weight (since $I_n$ has the smallest absolute weight). Similarly, $A^*$ can not include both $I_{n-1}$ and $I_{n+1}$ but not $I_n$; if so then $A^*$ could be improved by (i) using a single sensor in place of the two that cover $I_{n-1}$ and $I_{n+1}$, adding $I_n$ to $A^*$, and (ii) redeploying the sensor we have saved to either split one existing sensor by removing a negative-weight $I_m$ with $|w(I_m)|>|w(I_n)|$, or adding a new positive-weight $I_m$ with $|w(I_m)|>|w(I_n)|$. The net outcome is an improved total weight. We have shown that $A^*$ includes either all or none of $I_{n-1}\cup I_n\cup I_{n+1}$. Since $A^*$ is optimal for $\mathcal{I}$, and the restriction to $\tilde{\mathcal{I}}$ does not prevent AS-IM from finding this optimal $A^*$, it follows that $\tilde{A}^*=A^*$.

\paragraph{Case 2: $w(I_n)>0$.} Under the optimal solution $A^*$, a sensor cannot have a negative-weighted interval as an end interval, since dropping the negative-weight interval only increases the total weight. Furthermore, a sensor cannot include $I_n$ as an end interval of a series of intervals, since then the total weight could be improved by stopping sensing both $I_n$ and its sensed neighbour. Thus if $I_n$ is included in $A^*$ then either a sensor is observing only $I_n$, or a single sensor observes all of $I_{n-2}\cup I_{n-1}\cup I_n \cup I_{n+1}\cup I_{n+2}$. As in Case 1, if a sensor is observing only $I_n$ we can improve on $A^*$ by redeploying this sensor to either sense a better interval, or stop sensing an interval which has a higher negative weight than is lost by stopping sensing $I_n$. So again, under $A^*$, $I_n$ is either sensed with all its neighbours, or none of them are sensed. The same logic as in Case 1 ensures $\tilde{A}^*=A^*$.

\paragraph{Complexity:}
AS-IM requires sorting the $N$ initial intervals. Noticing that there are at most $N$ mergings, and assuming constant complexity for each merging, AS-IM offers an $O(N\log N)$ sample complexity. Since $N\le K_t$, AS-IM has a sample complexity not bigger than $O(K_t\log K_t)$.

\section{Discretisation error under linear and cubic root rates}
The effect of the different rates on the unavoidable discretisation error is depicted in Figure~\ref{fig.syn.rebinrates_inst_regret}. The regret for the linear rate is reduced at a faster rate than for the cubic root rate as the number of bins is increased at a much faster rate. However as we show in the main paper (Section 5.1) the other part of the regret due to error in action selection from the model forecast is much higher under the linear regret rate.

\begin{figure}[htbp]
\centering
\includegraphics[width=0.4\columnwidth]{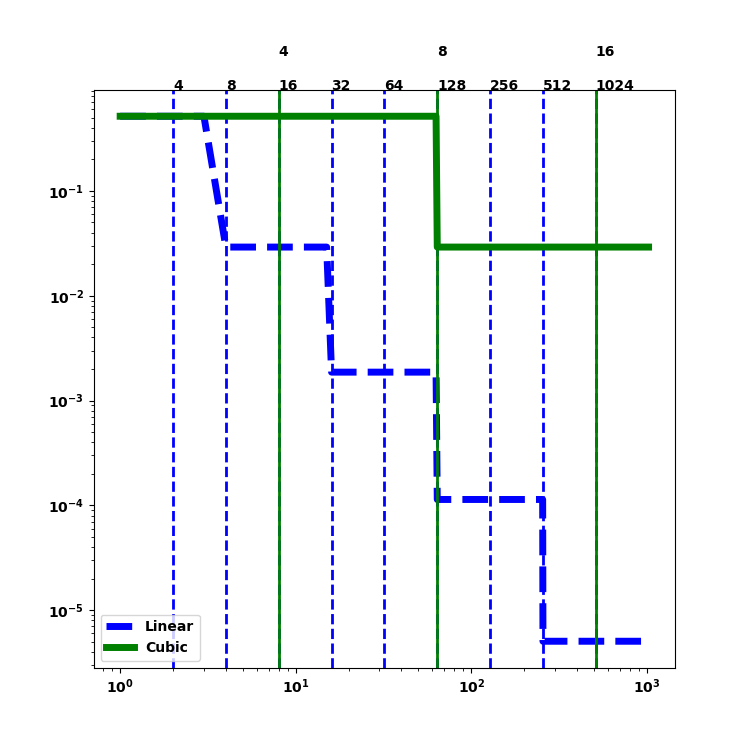}
 \caption{Instantaneous regret comparing linear and cube root rebinning rates. The vertical lines depict the rebinning times for the two different rate schedules. The time step (horizontal axis) and the regret (vertical axis) are both on a log scale. The number of bins for each rebinning rate are shown on the top horizontal axis. }\label{fig.syn.rebinrates_inst_regret}
\end{figure}

\section{Baselines used in the empirical study}
In the paper we have compared the TS approach other approaches which we now describe in more details.

\begin{enumerate}
    \item {\it UCB} approach, which is based on the FP-CUCB algorithm of~\cite{Grant2018} and requires the specification of an upper bound on the rate which we fix to the correct value in our experiments; in practise a conservative estimate is usually available. This is described in Algorithm \ref{UCBalg}.

\begin{algorithm}\label{UCBalg}
    \caption{UCB}
    \label{alg::UCB}
    \hrule
    \vspace{0.2cm}
    \textbf{Inputs:} Upper bound $\lambda_{\max} \geq \max_{x \in [0,1]} \lambda(x)$
    
    \textbf{Initialisation Phase:} For $t=1$ \begin{itemize}
        \item Select $A=[0,1]$
    \end{itemize}
    \textbf{Iterative Phase:} For $t\geq 2$
    \begin{itemize}
        \item For each $k \in \{1, \ldots, K_t\}$, evaluate $H_{k,t}(t-1)$ and $N_{k,t}(t-1)$ and calculate an index \begin{displaymath}
		\bar{\psi}_{k,t}= \frac{H_{k,t}(t-1)}{\Delta_tN_{k,t}(t-1)}+\frac{2\log(t)}{\Delta_t N_{k,t}(t-1)}+\sqrt{\frac{6\lambda_{\max}\log(t)}{\Delta_t N_{k,t}(t-1)}}.
		\end{displaymath}
		\item Choose an action $A_t$ that maximises $r(A)$ conditional on the true rate being given by the $\bar{\psi}_{k,t}$ values
		\item Observe the events in $A_t$
        \end{itemize}
        \hrule
        \vspace{0.2cm}
\end{algorithm}

\item A modified-UCB approach ({\it mUCB}) which has the same form as Algorithm 1 except $\lambda_{\max}$ is replaced with the empirical mean. Note this modification breaks the upper bound regret guarantee. 
The indices are :
\begin{displaymath}
		\bar{\psi}_{k,t}= \hat{\psi}_{k,t}(t-1)+\frac{2\log(t)}{\Delta_t N_{k,t}(t-1)}+\sqrt{\frac{6\hat{\psi}_{k,t}(t-1)\log(t)}{\Delta_t N_{k,t}(t-1)}}, \enspace k \in [K_t]
		\end{displaymath}
where $\hat{\psi}_{k,t}(t-1)=\frac{H_{k,t}(t-1)}{\Delta_tN_{k,t}(t-1)}$.
\item An {\it $\epsilon$-Greedy} approach where with probability $1-p_\epsilon$ an action $A_t$ is selected that maximises $r(A)$ conditional on the rate being given by the empirical mean values $\hat{\psi}_{k,t}$. With probability $p_\epsilon$, the action is instead chosen by sampling rates $\tilde\psi_{k,t}$ from independent $Gamma(\alpha,\beta)$ priors. In our experiments we fix $p_\epsilon=0.01$.
\end{enumerate}

\end{document}